\documentclass[journal]{IEEEtran}
\usepackage{amsmath,amsfonts}
\usepackage{algorithmic}
\usepackage{algorithm}
\usepackage{array}
\usepackage[caption=false,font=normalsize,labelfont=sf,textfont=sf]{subfig}
\usepackage{textcomp}
\usepackage{stfloats}
\usepackage{url}
\usepackage{verbatim}
\usepackage{graphicx}
\usepackage{algorithm}
\usepackage{algorithmic}
\usepackage{amsthm}

\usepackage{amsthm}
\usepackage{amsfonts,amssymb}
\usepackage{amsmath}
\usepackage{amssymb}
\usepackage{cancel}
\usepackage{dutchcal}

\usepackage{lineno}
\usepackage{color}
\usepackage[numbers]{natbib}

\newtheorem{lemma}{Lemma}

\usepackage{amsthm}

\usepackage{amsfonts,amssymb}
\usepackage{amsmath}
\usepackage{bm}
\usepackage{amsfonts}
\usepackage{mathtools}

\usepackage{adjustbox}

\newtheorem{remark}{Remark}

\newtheorem{theorem}{Theorem}

\begin{document}
	
	\title{Off-OAB: Off-Policy Policy Gradient Method with Optimal Action-Dependent Baseline}
	\author{IEEE Publication Technology,~\IEEEmembership{Staff,~IEEE,}
		% <-this % stops a space
		\thanks{This paper was produced by the IEEE Publication Technology Group. They are in Piscataway, NJ.}% <-this % stops a space
		\thanks{Manuscript received April 19, 2021; revised August 16, 2021.}}
	\author{Wenjia Meng, Qian Zheng, Long Yang, Yilong Yin, and Gang Pan
		\thanks{W. Meng and Y. Yin are with the School of Software, Shandong University, Jinan 250000, China (e-mail: wjmeng@sdu.edu.cn; ylyin@sdu.edu.cn).}
		\thanks{L. Yang is with the School of Artificial Intelligence, Peking University, China (e-mail: yanglong001@pku.edu.cn).}
		\thanks{Q. Zheng and G. Pan are with The State Key Lab of Brain-Machine Intelligence, College of Computer Science and Technology, Zhejiang University, Hangzhou 310000, China  (e-mail: qianzheng@zju.edu.cn; gpan@zju.edu.cn). (Corresponding author: Gang Pan)}
		\thanks{\copyright~20xx IEEE. Personal use of this material is permitted. Permission from IEEE must be obtained for all other uses, in any current or future media, including reprinting/republishing this material for advertising or promotional purposes, creating new collective works, for resale or redistribution to servers or lists, or reuse of any copyrighted component of this work in other works.}
	}
	\markboth{Journal of \LaTeX\ Class Files,~Vol.~14, No.~8, August~2021}%
	{Shell \MakeLowercase{\textit{et al.}}: A Sample Article Using IEEEtran.cls for IEEE Journals}
	
	%	\IEEEpubid{0000--0000/00\$00.00~\copyright~2021 IEEE}
	\maketitle
	
	\begin{abstract}
		Policy-based methods have achieved remarkable success in solving challenging reinforcement learning problems. Among these methods, off-policy policy gradient methods are particularly important due to that they can benefit from off-policy data. However, these methods suffer from the high variance of the off-policy policy gradient (OPPG) estimator, which results in poor sample efficiency during training. In this paper, we propose an off-policy policy gradient method with the optimal action-dependent baseline (Off-OAB) to mitigate this variance issue. 
		Specifically, this baseline maintains the OPPG estimator's unbiasedness while theoretically minimizing its variance.  To enhance practical computational efficiency, we design an approximated version of this optimal baseline. Utilizing this approximation, our method (Off-OAB) aims to decrease the OPPG estimator's variance during policy optimization. 
		We evaluate the proposed Off-OAB method on six representative tasks from OpenAI Gym and MuJoCo, where it demonstrably surpasses state-of-the-art methods on the majority of these tasks.
	\end{abstract}
	
	\begin{IEEEkeywords}
		Deep reinforcement learning,  policy-based method, off-policy reinforcement learning, policy learning.
	\end{IEEEkeywords}

	% page as needed:{\Large }
	% \fi
	
	\section{Introduction}
	\label{submission}
	\IEEEPARstart{D}{eep} reinforcement learning has witnessed substantial success in a variety of challenging areas, notably in games \cite{mnih2015human, silver2017mastering, papini2018stochastic, badia2020agent57, schrittwieser2020mastering, nikishin2024deep}, robotics \cite{levine2016end, haarnoja2018soft, hu2023graph}, and control tasks \cite{schulman2015trust, schmitt2020off,yang2022policy, luo2023human, meng2021off, chen2023hierarchical}. These reinforcement learning methods can be roughly divided into two categories: value-based methods and policy-based methods \cite{mnih2015human, lillicrap2015continuous, kallus2020statistically, fujita2018clipped, gruslys2018the, flet2021learning}. Unlike value-based methods \cite{mnih2015human}, policy-based methods \cite{sutton2000policy, saxena2023off, zhou2023robust} directly learn the policy distribution from samples to overcome the curse of dimensionality in action spaces \cite{sutton2018reinforcement, schulman2015trust, fujimoto2023sale}. However, policy-based methods suffer from high variance of policy gradient estimator, which necessitates a large number of samples to obtain accurate gradient estimator \cite{degris2012off, wang2016sample, gu2017interpolated,ciosek2020expected,kuba2021settling, barakat2023reinforcement}. This requirement for extensive interactions, often associated with high costs, leads to poor sample efficiency and inefficiencies in policy learning \cite{wang2016sample, GuLGTL17, gu2017interpolated, ciosek2020expected}.

	To address this issue, several methods are proposed to reduce the high variance of policy gradient estimator in policy-based methods \cite{mnih2016asynchronous,grathwohl2018backpropagation,yang2021sample,yang2022policy,dalal2023softtreemax,lin2023sample}. The majority of these variance reduction methods can be roughly divided into four categories according to their variance reduction ways. 
	Specifically, in the first category, the variance reduction methods use multi-step return and function approximation to estimate action value to reduce the variance of policy gradient estimator (e.g., \cite{mnih2016asynchronous, wang2016sample, espeholt2018impala}). In the second category, the variance reduction methods add deterministic gradient information to construct a new policy gradient estimator to reduce the variance of the policy gradient estimator (e.g.,\cite{GuLGTL17, gu2017interpolated,yang2022policy}). In the third category, the variance reduction methods introduce a state-dependent baseline into the policy gradient estimator to reduce its variance \cite{greensmith2004variance, degris2012off}. In the fourth category, the variance reduction methods introduce an action-dependent baseline into the policy gradient estimator to reduce its variance \cite{liu2018action, grathwohl2018backpropagation, wu2018variance}.

	Variance reduction methods for on-policy policy-based methods, which rely on on-policy data for policy optimization, encompass all four previously mentioned techniques \cite{GuLGTL17, greensmith2004variance, liu2018action, grathwohl2018backpropagation, wu2018variance}. In contrast, variance reduction strategies for off-policy policy-based methods, utilizing off-policy data, include only the first three techniques and do not employ action-dependent baselines \cite{wang2016sample, espeholt2018impala, gu2017interpolated, degris2012off}.
	However, the action-dependent variance reduction method is essential for the off-policy policy gradient (OPPG) estimator. It leverages the action-dependent baseline to accurately predict average policy performance, which can reduce the estimator's variance partly caused by challenges in assigning credit to actions \cite{liu2018action, grathwohl2018backpropagation, wu2018variance}.

	To overcome the above limitation, we propose an off-policy policy gradient method with the optimal action-dependent baseline, which we abbreviate as Off-OAB. This method innovatively incorporates an action-dependent baseline that includes action information into the off-policy policy gradient (OPPG) estimator, effectively reducing its variance. The use of an action-dependent baseline in our method is inspired by its application in the on-policy policy gradient estimator, as described in \cite{wu2018variance}. However, a distinguishing feature of our method is its capability to leverage off-policy data for policy optimization. This pivotal advancement allows for a reduction of on-policy interactions between the agent and the environment, which enhances the efficiency of policy learning by utilizing previously collected data.

	Specifically, we propose an action-dependent baseline that does not introduce bias into the OPPG estimator theoretically. With this unbiased action-dependent baseline, we derive its optimal formulation, which theoretically enables the OPPG estimator to achieve minimal variance. We further highlight the superiority of the optimal action-dependent baseline by theoretically demonstrating the variance reduction compared to the optimal state-dependent baseline. Next, we approximate the optimal action-dependent baseline for the computational benefit in practice. Finally, we propose an off-policy policy gradient method with this action-dependent baseline (Off-OAB) to reduce the variance of the OPPG estimator. Our contributions are as follows: 
	\begin{itemize}
		\item We propose an action-dependent baseline for the off-policy policy gradient (OPPG) estimator and present its optimal formulation. This approach is theoretically designed to minimize the variance of the OPPG estimator. 
		\item Following the theoretical development, we approximate the optimal action-dependent baseline for practical implementation. Building on this approximation, we introduce an off-policy policy gradient method, named Off-OAB, which incorporates this action-dependent baseline to effectively reduce the variance of the OPPG estimator. 
		\item Experiments conducted on continuous control tasks validate that our method outperforms state-of-the-art reinforcement learning methods on most tasks. 
	\end{itemize}

	\section{Preliminaries \& Background}
	
	This section provides the preliminaries for reinforcement learning and off-policy reinforcement learning, including the off-policy policy gradient theorem and off-policy actor-critic.

	\textbf{Reinforcement Learning}.
	Reinforcement learning (RL) \cite{sutton2018reinforcement} is formulated as \emph{Markov decision process} (MDP) expressed as $(S, A, P, r,  d_{0}, \gamma)$, which cosists of the state space $S$, the action space $A$, the transition dynamics distribution $P: S\times A \times S \rightarrow \mathbb{R}$, the reward function $r: S \times A \rightarrow \mathbb{R}$, the initial state distribution $ d_{0}: S \rightarrow \mathbb{R}$, the discount factor $\gamma  \in (0, 1)$. 
	In MDP, the agent interacts with the environment according to the policy $\pi : S \times A \rightarrow [0,1]$ to generate a trajectory, \emph{i.e.}, $\tau = (s_0, a_0, r_0, \cdots, s_T, a_T, r_T)$, where $a_t \sim \pi(\cdot|s_t)$, and $s_{t+1} \sim P(\cdot |s_t, a_t) $. Return $R_t$ is the accumulated discounted reward from timestep $t$, $R_t = \sum^{\top}_{k=t} \gamma^{k-t} r(s_k, a_k)$. 
	$V_{\pi}(s_t) $ and $Q_{\pi}(s_t, a_t)$ are  defined as:
	\begin{flalign}
		V_{\pi}(s_t) &= \mathbb{E}_{a_t, s_{t+1}, \cdots}[R_t], \\
		Q_{\pi}(s_t, a_t) &= \mathbb{E}_{s_{t+1}, a_{t+1}, \cdots}[R_t].
	\end{flalign} 
	%Note that $V_{\pi}(s_t)$ and $Q_{\pi}(s_t, a_t)$ satisfy: $V_{\pi}(s_t)=\sum_{a_t \in A }\pi(s_t, a_t)  Q_{\pi}(s_t, a_t)$. 

	\textbf{Off-Policy Learning}. This paper considers the off-policy reinforcement learning setting, where data is generated by \emph{behavior policy} $\mu$ to optimize the \emph{target policy} $\pi$. We aim to maximize the objective function,
	\begin{flalign}
		J(\pi_{\theta} )  =\sum_{s \in S}  d_{\mu}(s) \sum_{a\in A} \pi_{\theta} (a|s) Q_{\pi_{\theta}}(s,a),  \label{off-policy objective function}
	\end{flalign}
	where $\pi_{\theta}$ denotes the target policy parameterized by $\theta$, $d_{\mu} = \lim_{t \rightarrow \infty} P(s_t=s|s_0, \mu)$ denotes the limiting state distribution under behavior policy $\mu$. $P(s_t=s|s_0, \mu)$ denotes the likelihood of reaching state $s_t =s$ from an initial state $s_0$ under the behavior policy $\mu$.
	
	\textbf{Off-Policy Policy Gradient Theorem}.
	To maximize the objective $J(\pi_{\theta})$ in Eq. (\ref{off-policy objective function}), off-policy policy gradient theorem \citep{degris2012off} updates the policy parameter $\theta$ by,
	\begin{flalign}
		\theta_{k+1}  = \theta_{k} + \alpha_{k} \nabla_{\theta} J(\pi),
	\end{flalign}
	where $k$ is the step index, $\alpha_{k}$ is the step size. For practice, off-policy policy gradient theorem \cite{degris2012off} considers to approximate $\nabla_{\theta} J(\pi)$ as follows
	\begin{flalign}
		%\label{approximated_gradient}
		\begin{aligned}
			&\nabla_{\theta} J(\pi)  \\
			&\approx \mathcal{g}_{\text{off}}= \mathbb{E}_{s\sim d_{\mu}, a \sim \mu(a|s)} \left[ \rho(s, a) \nabla_{\theta}\log\pi(a|s) Q_{\pi}(s,a)   \right]\label{expectation formulation-01}
		\end{aligned}
	\end{flalign}
	where $\rho(s, a)=\frac{\pi(a|s)}{\mu(a|s)}$. We denote the $\mathcal{g}_{\text{off}}$ in Eq. (\ref{expectation formulation-01}) as off-policy policy gradient (OPPG) estimator. Considering the data $\{s_{t},a_{t},r_{t}\}_{t=0}^{T}$ generated under the behavior policy $\mu$, we can derive the empirical OPPG estimator using Monte Carlo samples,
	\begin{flalign}
		\label{expectation formulation}
		\hat{\mathcal{g}}_{\text{off}}=\frac{1}{T}\sum_{t=0}^{T}
		\rho(s_t, a_t)\nabla_{\theta}\log\pi(a_t|s_t) Q_{\pi}(s_t,a_t).
	\end{flalign}
	
	The empirical OPPG estimator in Eq. (\ref{expectation formulation}) suffers from high variance \citep{wang2016sample, ni2022optimal, zhang2021convergence}, which is caused by both Monte Carlo techniques \cite{sutton2018reinforcement} and importance sampling ratio $\rho(s_t, a_t)$. High variance in the OPPG estimator negatively impacts the effectiveness of off-policy policy gradient methods in real-world applications.

	\textbf{Off-Policy Actor-Critic}.
	To mitigate the high variance issue encountered in off-policy policy gradient methods, the off-policy actor-critic (Off-PAC) \cite{degris2012off} introduces a state-dependent baseline to the OPPG estimator to reduce its variance. 
	Specifically, by incorporating a state-dependent baseline ($V_{\pi}(s)$) into the OPPG estimator, Off-PAC reduces variance without adding bias, 
	\begin{flalign}	\label{state dependent baseline}
		\begin{aligned}
			\mathcal{g}_{\text{off}}(V_{\pi}(s))
			= \mathbb{E}_{s\sim d_{\mu}, a \sim \mu(a|s)} \left[ \rho(s, a) \nabla_{\theta}\log\pi(a|s) A_{\pi}(s,a) \right],
		\end{aligned}
		\vspace{-10pt}
	\end{flalign}
	where $ A_{\pi}(s,a)=Q_{\pi}(s,a)  - V_{\pi}(s)$. 
	
	However, the state-dependent baseline cannot use the information of action, which prevents the Off-PAC from predicting accurate average performance and achieving lower variance. Addressing this limitation is the central aim of our paper, wherein we explore how incorporating action information can efficiently reduce the variance introduced by Monte Carlo techniques and the importance sampling ratio in the off-policy policy gradient estimator.

	\section{Action-Dependent Baseline for Off-Policy Policy Gradient (OPPG) Estimator}
	In this section, we propose an action-dependent baseline for the OPPG estimator, designed to reduce its variance by utilizing the additional action information. Concretely, we initially propose an action-dependent baseline that does not bring bias for the OPPG estimator. Subsequently, we derive the optimal action-dependent baseline theoretically, guiding the design of the OPPG estimator in our method. Finally, we demonstrate the superiority of the optimal action-dependent baseline by theoretically highlighting the reduced OPPG variance compared to the optimal state-dependent baseline.

	\subsection{Unbiased Off-Policy Action-Dependent Baseline}
	
	In this paper, we explore a widely adopted policy parameterization characterized by a multivariate Gaussian with diagonal covariance, reflecting the rich internal structures commonly found in practice (see \citep[Chapter 13]{sutton2018reinforcement}). This policy parameterization means that each $i$-th component $a^i$ of the action $a$ is conditionally independent of other components at any given state $s$. Leveraging the policy's conditional independence, in an $m$-dimensional action space, the policy $\pi(a|s)$ is represented as follows,
	\begin{flalign}
		\pi(a|s) = \prod_{i=1}^{m} \pi(a^i |s),
	\end{flalign}
	where each $\pi(a^i |s)$ represents the policy value for the $i$-th action component $a^i$. 
	Using this parameterization, we reformulate the OPPG as shown in Eq. (\ref{expectation formulation-01}),
	\begin{flalign}\label{gradient estimator with multivariate gaussian}
		\mathcal{g}_{\text{off}}&=\mathbb{E}_{s\sim d_{\mu}, a\sim  \mu(a|s)} \left[\sum_{i=1}^{m} \rho(s, a) \nabla_{\theta}\log\pi(a^i|s)  Q_{\pi}(s,a)   \right]. 
	\end{flalign}
	The derivation of Eq. (\ref{gradient estimator with multivariate gaussian}) is detailed in Appendix \ref{gradient_estimator_appendix}.

	Eq. (\ref{gradient estimator with multivariate gaussian}) suggests the OPPG estimator is the sum of $m$ factors. 
	For the $i$-th factor in Eq. (\ref{gradient estimator with multivariate gaussian}), we use an action-dependent baseline $b_i$. 
	To ensure the proposed baseline remains unbiased for the OPPG estimator, we introduce an action-dependent baseline $b_{i}(s, a^{-i})$ for the $i$-th factor, where $a^{-i}$ represents all action components excluding the $i$-th component. 
	The action-dependent baseline $b_{i}(s, a^{-i})$ is independent of the $i$-th action and depends on the state and all other actions.

	Using the action-dependent baseline $b_{i}(s, a^{-i})$, we define the OPPG estimator with $b_{i}(s, a^{-i})$ as follows, 
	\begin{flalign}
		\label{non-biased estimator}  
		\mathcal{g}_{\text{off}}(b)
		:= \mathbb{E}_{s\sim d_{\mu}, a\sim \mu}\left[\sum_{i=1}^m \mathcal{g}_{\text{off}}^i(b) \right],
	\end{flalign}
	each $\mathcal{g}_{\text{off}}^i(b) $ is defined as follows,
	\begin{flalign}\label{i_th_baseline}
		\mathcal{g}_{\text{off}}^i(b)  = \rho(s, a) \nabla_{\theta}\log\pi(a^i|s)  ( Q_{\pi}(s,a) - b_i(s, a^{-i})),
	\end{flalign}
	where $	\rho(s, a) =\frac{\pi(a|s)}{\mu(a|s)}$, $m$ is the number of action dimensions.

	\begin{remark}[Unbiasedness of Action-dependent Baseline]\label{remark_unbiasedness}
		The action-dependent baseline in the proposed OPPG estimator, as shown in Eq. (\ref{non-biased estimator}), maintains the unbiased nature of the original OPPG estimator in Eq. (\ref{gradient estimator with multivariate gaussian}), supported by the fact that,
		\begin{flalign}\label{action_dependent_equation}
			\mathbb{E}_{a \sim \mu(a|s)}\left[ \rho(s, a)\nabla_{\theta}\log\pi(a^i |s) b_{i}(s, a ^{-i} ) \right]=0,
		\end{flalign}
		which implies the following fact
		\begin{flalign}\label{unbiased_oppg_estimator}
			\mathcal{g}_{\text{off}} =\mathcal{g}_{\text{off}}(b).
		\end{flalign}
	\end{remark}
	\begin{proof}
		See Appendix \ref{Unbiassness}.
	\end{proof}

	\subsection{Optimal Off-Policy Action-Dependent Baseline}
	
	Using the above baseline $b_{i}(s, a^{-i})$, we derive the optimal action-dependent baseline that minimizes OPPG variance in this section, \emph{i.e.}, 
	\begin{flalign}
		b^{\star}(s,a)=\arg\min_{b}\mathrm{Var}\left[\mathcal{g}_{\text{off}}(b)\right].
	\end{flalign}
	For clarity and brevity, we denote $b^{\star}(s,a)$ as $b^{\star}$.
	Furthermore, we decompose the optimal baseline $b^{\star}$ as follows, 
	\begin{flalign}
		b^{\star}=\bigcup_{i\in \{1,2,\cdots,m\}}\{b^{\star}_{i}(s,a^{-i})\}.
	\end{flalign}

	\begin{theorem}
		[Optimal Off-Policy Action-Dependent Baseline]
		\label{theorem-optimal-baseline}
		Let $\mathcal{g}_{\text{off}}(b)$ be the off-policy policy gradient estimator defined in Eq. (\ref{non-biased estimator}). The optimal off-policy action-dependent baseline that minimizes the variance of $\mathcal{g}_{\text{off}}(b)$ is
		\begin{flalign}
			\label{them-baseline-01}
			b^{\star}_i(s, a^{-i}) =\frac{\mathbb{E}_{a^i \sim \mu} \left[ \lVert \rho(s, a) \nabla_{\theta}\log\pi(a^i|s) \rVert^2 Q_{\pi}(s,a) \right]}{\mathbb{E}_{a^i \sim \mu} \left[ \lVert \rho(s, a) \nabla_{\theta}\log\pi(a^i|s) \rVert^2 \right]},
		\end{flalign}
		where $\rho(s, a)=\frac{\pi(a|s)}{\mu(a|s)}$.
	\end{theorem}
	\begin{proof}
		See Appendix \ref{optimal_action_dependent_baseline_proof}.
	\end{proof}
	
	Theorem \ref{theorem-optimal-baseline} provides a theoretical foundation for optimal action-dependent baseline in OPPG estimator, as shown in Eq. (\ref{non-biased estimator}), by focusing on variance reduction.

	\subsection{Variance Reduction Compared to the Optimal State-Dependent Baseline}
	
	To further highlight the benefits of the optimal action-dependent baseline from Theorem \ref{theorem-optimal-baseline}, we evaluate the variance difference in OPPG between the optimal state and action-dependent baseline, demonstrating the optimal action-dependent baseline's effectiveness in reducing variance compared to the state-dependent baseline.

	To clarify this variance difference in OPPG, we initially present the OPPG estimator incorporating the optimal action-dependent baseline ($b^{\star}_i(s, a^{-i}) $) for the $i$-th action dimension,
	\begin{flalign}
		\mathcal{g}_{\text{off}}^i(b^{\star}_i(s, a^{-i}))  = \rho(s, a) \nabla_{\theta}\log\pi(a^i|s)  ( Q_{\pi}(s,a) - b^{\star}_i(s, a^{-i})).
	\end{flalign}
	Next, we introduce the OPPG estimator with the optimal state-dependent baseline ($b^{\star}(s)$) for the $i$-th action dimension,
	\begin{flalign}
		\mathcal{g}_{\text{off}}^i(b^{\star}(s))  = \rho(s, a) \nabla_{\theta}\log\pi(a^i|s)  ( Q_{\pi}(s,a) - b^{\star}(s)).
	\end{flalign}
	Using the described OPPG estimators, we define the variance difference between them,
	\begin{flalign}\label{variance_diff_definition}
		\small
		\Delta \mathrm{Var}(\mathcal{g}_{\text{off}}^i(b^{\star}(s))) \triangleq\mathrm{Var}\left[ \mathcal{g}_{\text{off}}^i(b^{\star}(s)) \right]- \mathrm{Var} \left[ \mathcal{g}_{\text{off}}^i(b^{\star}_i(s, a^{-i})) \right]. 
	\end{flalign}
	This variance difference is analyzed in Theorem \ref{excess_variance_theorem}.

	\begin{theorem}[Variance Difference between Optimal State and Action-Dependent Baseline]\label{excess_variance_theorem}
		Let the variance difference $\Delta \mathrm{Var}(\mathcal{g}_{\text{off}}^i(b^{\star}(s))) $ between optimal state and action-dependent baseline for OPPG estimator be defined in Eq. (\ref{variance_diff_definition}). This variance difference satisfies,
		\begin{flalign}
			\begin{aligned}
				& \Delta \mathrm{Var}(\mathcal{g}_{\text{off}}^i(b^{\star}(s)))\\
				&= \mathbb{E}_{s\sim  d_{\mu}, a^{-i}\sim \mu}\Big[ \Big( b^{\star}(s)- b^{\star}_i(s, a^{-i}) \Big)^2 \\
				&\qquad\qquad\qquad\qquad \mathbb{E}_{a^i \sim \mu}\left[ \lVert \rho(s, a) \nabla_{\theta}\log\pi(a^i|s) \rVert^2 \right]\Big].
			\end{aligned}
		\end{flalign}
	\end{theorem}
	\begin{proof}
		See Appendix \ref{optimal_state_dependent_excess_variance}. 
	\end{proof}
	
	Theorem \ref{excess_variance_theorem} highlights the variance reduction achieved by the optimal action-dependent baseline compared to the state-dependent baseline, which indicates its superior efficacy in reducing OPPG variance. 
	This optimal action-dependent baseline formulated in Theorem \ref{theorem-optimal-baseline} is the foundation for us to design the following algorithm.

	\section{Proposed Off-OAB}

	In this section, building on the baseline outlined in Theorem \ref{theorem-optimal-baseline}, we introduce an approximated optimal baseline for the OPPG estimator to lower computational costs. Besides, we demonstrate its proximity to the optimal baseline, particularly when the policy factor weakly correlates with the action value. 
	Using this approximated baseline, we propose an off-policy policy gradient method incorporating this optimal action-dependent baseline, named Off-OAB, and detail its implementation.

	\begin{algorithm}[b]
		%	\scriptsize
		\caption{Off-OAB Method}
		\label{alg:variance reduction method}
		\begin{algorithmic}
			\STATE \textbf{Input:} Environment \emph{E}, batch size $B$, discount factor $\gamma$, replay buffer $\mathcal{B}$, total timesteps $T$, learning rate $\lambda_{Q}$, $\lambda_{\pi}$, decay rate for critic target network $\tau$;
			\STATE Initialize critic network $Q_{w}$ and actor network $\pi_{\theta}$ with random parameters $w$, $\theta$;\\
			\STATE  Initialize target critic network $\bar{w} \leftarrow w$; replay buffer $\mathcal{B}$;
			%			\STATE Initialize replay buffer $\mathcal{B}$;
			\FOR{each iteration}
			\STATE \textcolor{blue}{--- --- --- \texttt{Collect data} --- --- ---}
			\FOR{each environment step}
			\STATE Sample action from behavior policy: $a_t \sim \mu(a_t|s_t) $;
			\STATE Receive next state $s_{t+1}$ and reward $r(s_t,a_t)$ from \emph{E}: 
			\STATE \begin{center} $s_{t+1} \sim P(s_{t+1}|s_t, a_t) $; \end{center}
			\STATE Store the transition in the replay buffer: 
			\STATE 
			\begin{center}
				$\mathcal{B} \leftarrow \mathcal{B} \cup \{(s_t, a_t,r(s_t,a_t), s_{t+1}, \mu(a_t|s_t))\}; $
			\end{center} 
			\ENDFOR
			\FOR{each gradient step}
			\STATE  Perform update procedure in Algorithm \ref{update_procedure}.
			\ENDFOR
			\ENDFOR
		\end{algorithmic} 
	\end{algorithm}

	\subsection{Approximated Optimal Off-Policy Baseline}
	Recall the optimal action-dependent baseline in Theorem \ref{theorem-optimal-baseline},
	\begin{flalign}
		\label{tem-optimal-action-dependent-baseline}
		b^{\star}_i(s, a^{-i}) =\frac{\mathbb{E}_{a^i \sim \mu} \left[ \lVert \rho(s, a) \nabla_{\theta}\log\pi(a^i|s) \rVert^2 Q_{\pi}(s,a) \right]}{\mathbb{E}_{a^i \sim \mu} \left[ \lVert \rho(s, a) \nabla_{\theta}\log\pi(a^i|s) \rVert^2 \right]}.
	\end{flalign}
	However, practically implementing Eq. (\ref{tem-optimal-action-dependent-baseline}) faces the challenge of high computational cost. Calculating the optimal action-dependent baseline involves repeated computation of $\lVert \rho(s, a) \nabla_{\theta}\log\pi(a^i|s) \rVert^2 $, which is computationally demanding.

	To tackle this challenge, we approximate the optimal baseline in Eq. (\ref{tem-optimal-action-dependent-baseline}) by adopting a similar action-dependent baseline,
	\begin{flalign}\label{appro_b_s_a}
		b_i(s, a^{-i}) = \mathbb{E}_{a^i \sim \mu}[Q_{\pi}(s,a)]
	\end{flalign}
	for computational efficiency. 
	We choose this analogous baseline because it simplifies the computation process to the estimation of $Q_{\pi}(s,a)$. Moreover, this baseline closely approximates the optimal one, especially when the policy factor is weakly correlated with the action value, as demonstrated in Theorem \ref{theorem_Q_pi_sa}.
	\begin{theorem}[Close to Optimal Action-Dependent Baseline]\label{theorem_Q_pi_sa}
		Define the approximated baseline as $\mathbb{E}_{a^i \sim \mu}[Q_{\pi}(s,a)]$ in Eq. (\ref{appro_b_s_a}). 
		The variance of this approximated baseline is close to that of the optimal action-dependent baseline when the policy factor is weakly correlated with the action value,
		\begin{flalign}	
			\begin{aligned}
				&\Delta \mathrm{Var}( \mathcal{g}_{\text{off}}^i(\mathbb{E}_{a^i \sim \mu}[Q_{\pi}(s,a)])) \\
				&= \mathrm{Var}\left[ \mathcal{g}_{\text{off}}^i(\mathbb{E}_{a^i \sim \mu}[Q_{\pi}(s,a)]) \right]- \mathrm{Var} \left[ \mathcal{g}_{\text{off}}^i(b^{\star}_i(s, a^{-i})) \right] 
				\approx 0, \\
				& \text{when},\\
				& \quad \mathbb{E}_{a^i \sim \mu} \left[ \lVert \rho(s, a) \nabla_{\theta}\log\pi(a^i|s) \rVert^2 Q_{\pi}(s,a) \right] \\
				&\quad \approx \mathbb{E}_{a^i \sim \mu} \left[ \lVert \rho(s, a) \nabla_{\theta}\log\pi(a^i|s) \rVert^2  \right] \mathbb{E}_{a^i \sim \mu}[Q_{\pi}(s,a)].
			\end{aligned}
		\end{flalign}
	\end{theorem}
	\begin{proof}
		See Appendix \ref{excess_variance_Q_pi_sa}.
	\end{proof}

	\begin{algorithm}
		\caption{Actor and Critic Weights Update Procedure}\label{update_procedure}
		\begin{algorithmic}
			\STATE \textcolor{blue}{--- --- --- \texttt{Update critic weights} --- --- ---}
			\STATE Sample $N$ transitions  from $\mathcal{B}$: $\{(s_j, a_j, r_j, s'_j,\mu_j ), j=1 \cdots N \} $;
			\STATE Set critic update target:
			\STATE 
			\begin{center}
				$y_j \leftarrow r_j + \gamma \max_{a'} Q_{\bar{w}}(s'_j, a'_j) $, \\
				$\delta_{j} \leftarrow y_j - Q_{w}(s_j,a_j)$;
			\end{center}
			\STATE Update the critic network weights: 
			\[
			w \leftarrow w - \lambda_{Q} \nabla_{w}  \left( \frac{1}{N}\sum_{j=1}^{N} \delta_j^2  \right) ;
			\]
			\STATE Update target critic network weights:
			\[
			\bar{w} \leftarrow \tau w + (1-\tau) \bar{w};
			\]	
			\STATE \textcolor{blue}{--- --- --- \texttt{Update actor weights} --- --- ---}
			\STATE Calculate baseline $b_i(s_j, a^{-i}_j)$ by Eq. (\ref{appro_b_s_a}):
			\begin{center}
				$b_i(s_j, a^{-i}_j)= \mathbb{E}_{a^i \sim \mu}[Q_{w}(s_j,a_j)]$;
			\end{center}
			\STATE
			Set the importance ratio as $\rho_j=\frac{\pi_{\theta}(a_j|s_j)}{\mu(a_j|s_j)}$ and obtain the gradient by Eq. (\ref{i_th_baseline}):
			\STATE
			\begin{center}
				%				\small
				\begin{flalign}
					\mathcal{g}_{\text{off}}^i(b_j) = \rho_j \nabla_{\theta}\log\pi_{\theta}(a^i_j|s_j)  \left( Q_{w}(s_j,a_j) - b_i(s_j, a^{-i}_j) \right ); \nonumber
				\end{flalign}
			\end{center}
			\STATE Update policy weights using $\mathcal{g}_{\text{off}}^{i}(b_j)$:  
			\[ 
			\theta \leftarrow  \theta + \lambda_{\pi} \frac{1}{N}\sum_{j=1}^{N}  \left( \sum_{i=1}^m \mathcal{g}_{\text{off}}^i(b_j) \right) ;
			\]
		\end{algorithmic}
	\end{algorithm}

	\begin{figure*}[htb]
		\centering
		\includegraphics[width=18cm,height=10cm]{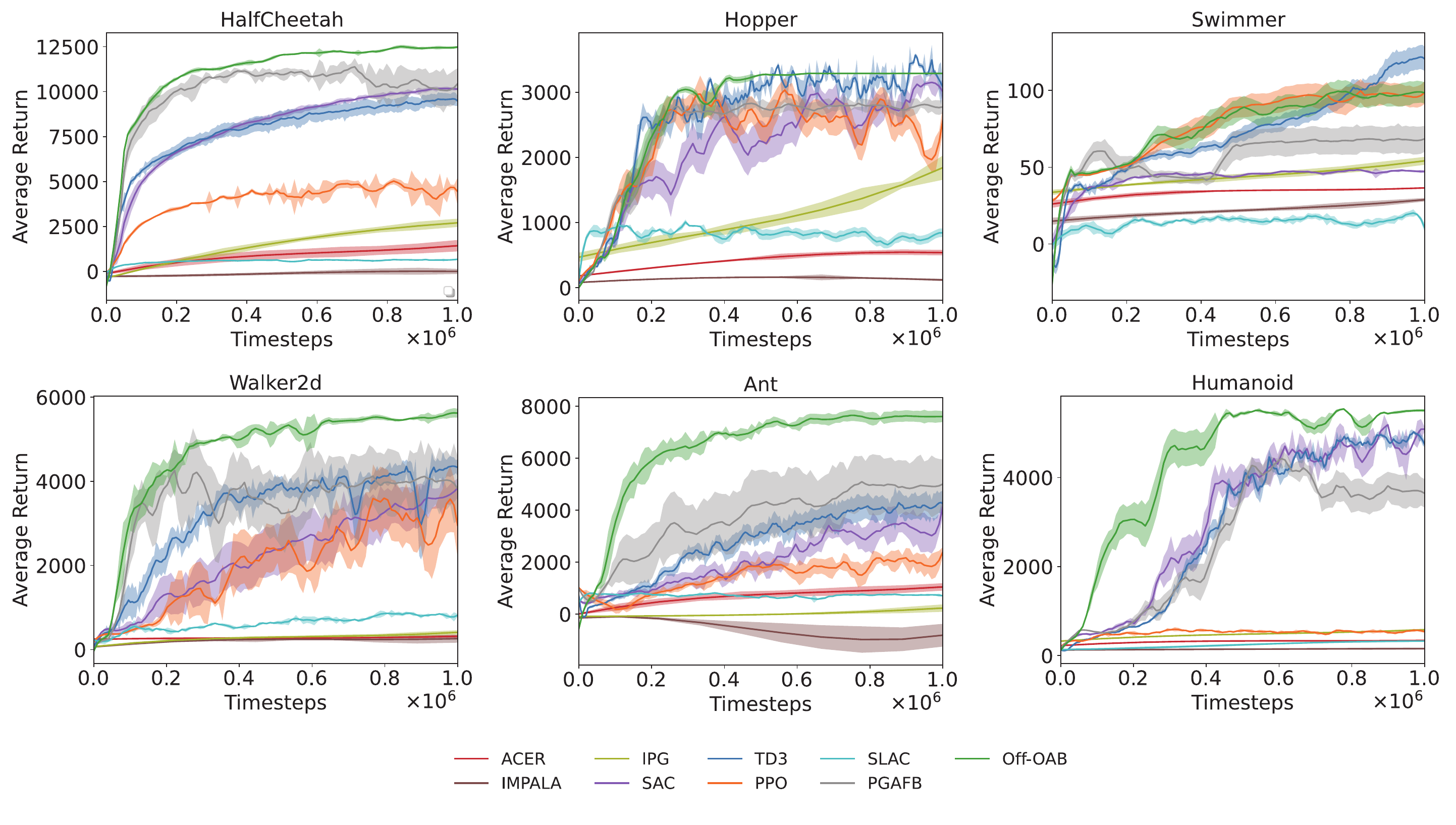}
		\vspace{-10pt}
		\caption{Results of proposed Off-OAB method and other state-of-the-art deep reinforcement learning methods (ACER, IMPALA, IPG, SAC, TD3, PPO, SLAC, and PGAFB) on representative tasks. The standard deviation over five seeded runs is denoted by the shaded region. The $X$-aixs and $Y$-axis separately denote environment timesteps and average return.}
		\label{State-of-the-art Methods}
	\end{figure*}

	\subsection{Detailed Algorithm Implementation for Off-OAB} 
	
	Using the approximated optimal baseline in Eq. (\ref{appro_b_s_a}), we propose an off-policy policy gradient method leveraging this optimal action-dependent baseline (Off-OAB) to reduce the OPPG estimator's variance. Algorithm \ref{alg:variance reduction method} details the complete process of the proposed Off-OAB.

	Algorithm \ref{alg:variance reduction method} starts by initializing the critic ($Q_{w}$) and actor ($\pi_{\theta}$) networks with random parameters $w$ and $\theta$, along with setting up the target critic network and replay buffer. 
	During each interaction, the agent collects data by interacting with the environment. It samples action $a_t$ from the behavior policy $\mu(a_t|s_t)$ at each state $s_t$ and timestep $t$. 
	Following this,  the environment provides the next state $s_{t+1}$ based on the transition dynamic probability $P(s_{t+1}|s_t, a_t) $, and a reward $r(s_t, a_t)$. The transition ($s_t, a_t, r(s_t, a_t), s_{t+1}, \mu(a_t|s_t)$) is then stored in the replay buffer $\mathcal{B}$.

	The learning process for the critic and actor networks is detailed in Algorithm \ref{update_procedure}. Specifically, we utilize Q-learning \cite{mnih2015human} to update the critic network by sampling $N$ transitions from replay buffer $\mathcal{B}$, which consist of  $\{(s_j, a_j, r_j, s'_j,\mu_j ), j=1 \cdots N \}$. The critic target is calculated by maximizing action values, which can be then used to update the critic network's weights. Subsequently, we align the target critic network's weights with the critic network's weights. With the critic network, we derive the action-dependent baseline $b_i(s, a^{-i})$ from Eq. (\ref{appro_b_s_a}) for the $i$-th action dimension. Utilizing this baseline, we compute the OPPG estimator as shown in Eq. (\ref{i_th_baseline}) to update the actor network's weights.

	\section{Experiments}	
	
	In this section, we conduct experiments on representative continuous control tasks to assess the efficacy of the proposed Off-OAB method utilizing the proposed action-dependent baseline. We first detail our experimental setup, including task descriptions, network configurations, and hyperparameters. We then evaluate the proposed Off-OAB by comparing it with state-of-the-art methods. We next compare the proposed action-dependent baseline with other baselines to validate the effectiveness of the proposed action-dependent baseline. We further assess the sample efficiency of our method by measuring the timesteps needed to achieve certain returns compared to other methods. Lastly, we evaluate our action-dependent baseline's effectiveness in reducing variance by analyzing policy gradient variances against other baselines.

	\subsection{Setup}
	This section outlines our experimental setup, including tasks, network setups, and hyperparameters. Experimental tasks consists of six representative continuous control tasks (\emph{HalfCheetah}, \emph{Hopper}, \emph{Swimmer}, \emph{Walker2d}, \emph{Ant},  \emph{Humanoid}) from OpenAI Gym \cite{brockman2016openai} and MuJoCo \cite{todorov2012mujoco}. 
	The experiments utilize policy and action value networks, each with two hidden layers of $256$ neurons. Key hyperparameters are:  a discount factor ($\gamma$) of $0.99$, a batch size of $256$, and a maximum of 1000 timesteps per episode. Learning rates for the critic and policy networks are set to $0.0003$, with a replay buffer size of $1$ million. The total timesteps for the experiments are $1$ million. Optimization is performed using Adam optimizer, and results are averaged over five runs with different seeds. Experiments are conducted on a server with four NVIDIA GeForce RTX 3090 GPUs and $512$GB of memory. For implementing the experimental methods, we used the code from the original authors as provided in their original papers. Details on hyperparameters and implementation \footnote{We will release our code when this paper is accepted.} are available in Appendix \ref{detailed_hy}.

	\begin{table*}[htb]
		\centering
		\caption{Results of the maximal average returns among the whole training process. We denote the best results among these methods in boldface.}
		\label{maximal average return comparison}
		%\resizebox{\textwidth}{13.2mm}{
		\begin{adjustbox}{width=1\textwidth}
			\begin{tabular}{l| c  c  c  c c c c  c || c}
				\hline
				Return & ACER & IMPALA & IPG & SAC & TD3 & PPO &  SLAC & PGAFB & Off-OAB\\
				\hline\hline
				HalfCheetah & $1426.7$ & $6.2$ & $2701.4$ & $10178.3$ &$9598.8$ & $5023.7$  & $672.0$ & $11384.2$ & \bm{$12506.0$}\\
				\hline
				Hopper &  $545.6$ & $162.7$ & $1842.4$ & $3150.7$ & \bm{$3568.6$} & $3033.4$ &  $1008.3$ & $2829.5$ &  $3289.9$ \\
				\hline
				Swimmer  & $36.5$  & $28.8$ & $54.2$ & $48.3$  & \bm{$121.8$} & $99.8$  & $20.2$ & $68.5$& $99.5$\\
				\hline
				Walker2d & $314.9$ & $259.6$ & $408.6$ & $3836.4$  & $4359.6$  & $3598.8$  &  $865.3$ & $4258.6$& \bm{$5628.0$}\\
				\hline
				Ant & $1047.2$ & $-98.0$ &  $232.4$  & $4116.2$  &  $4297.8$ & $2381.2$ &  $817.4$ &$5096.1$ & \bm{$7652.3$}\\
				\hline    
				Humanoid & $341.2$ & $156.3$& $581.4$ & $5184.7$ & $5024.8$ & $591.1$ & $325.9$& $4425.8$ &  \bm{$5539.8$}\\
				\hline    
			\end{tabular}
			%}
		\end{adjustbox}
	\end{table*}
	
	\begin{figure*}[htb]
		\centering
		\includegraphics[width=18cm,height=10cm]{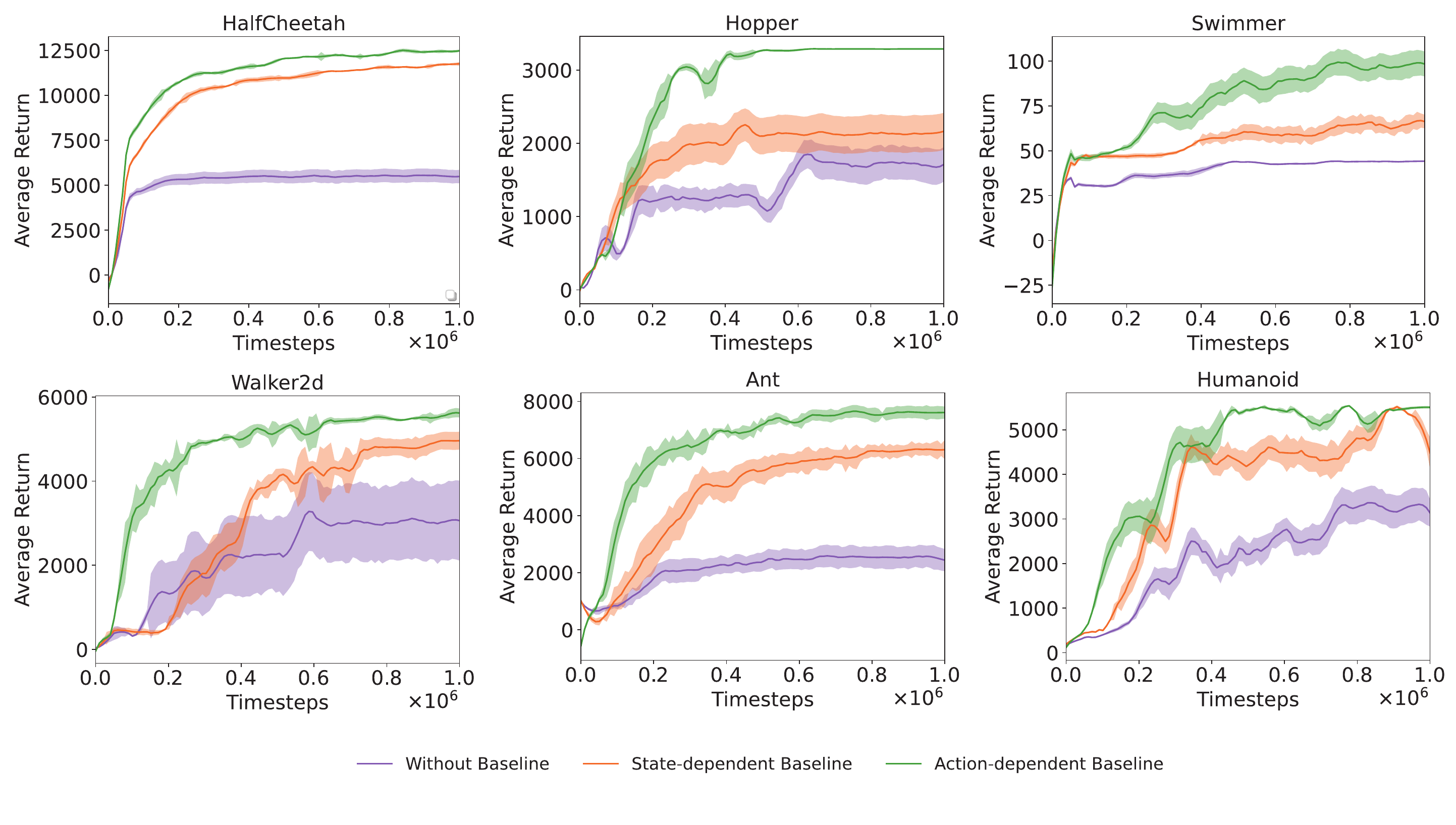}
		\vspace{-10pt}
		\caption{Results of our method with varying baselines. The standard deviation over five seeded runs is denoted by the shaded region. The $X$-aixs and $Y$-axis separately denote environment timesteps and average return.}
		\label{comparison_with_other_baselines}
	\end{figure*}

	\subsection{Comparison with the State-of-the-art Methods}
	In this section, we compare the proposed Off-OAB with state-of-the-art methods, namely ACER \cite{wang2016sample}, IMPALA \cite{espeholt2018impala}, IPG \cite{GuLGTL17}, SAC \cite{haarnoja2018soft}, TD3 \cite{fujimoto2018addressing}, PPO \cite{schulman2017proximal}, SLAC \cite{lee2020stochastic}, PGAFB\footnote{We refer to the method described in \cite{wu2018variance} as PGAFB for brevity.} \cite{wu2018variance}, to evaluate our method's performance.

	Figure \ref{State-of-the-art Methods} showcases the training curves comparing the proposed Off-OAB method with state-of-the-art methods: ACER, IMPALA, IPG, SAC, TD3, PPO, SLAC, and PGAFB. The results indicate that Off-OAB consistently achieves higher returns on several tasks, including HalfCheetah, Walker2d, Ant, and Humanoid, outperforming these state-of-the-art methods. Particularly on the Ant task, our method's performance significantly surpasses others. Moreover, Off-OAB reaches comparable or superior returns more efficiently, requiring fewer timesteps on most tasks. For instance, on the HalfCheetah, Walker2d, and Ant tasks, our method reaches stable performance within about $300000$ timesteps, showcasing greater efficiency compared to other methods.

	Table \ref{maximal average return comparison} presents the highest average returns achieved by our method and state-of-the-art methods during training. Our method consistently ranks the top one for maximum average returns on most tasks, notably HalfCheetah, Walker2d, Ant, and Humanoid. While our method didn't secure the highest returns on Hopper and Swimmer, the difference in returns between our method and the top-performing ones on these tasks is marginal. The results in Figure \ref{State-of-the-art Methods} and Table \ref{maximal average return comparison} strongly support our method's superior performance over state-of-the-art reinforcement learning methods (ACER, IMPALA, IPG, SAC, TD3, PPO, SLAC, and PGAFB) on most tasks.

	\begin{figure*}[htb]
		\centering
		\includegraphics[width=18cm,height=5.5cm]{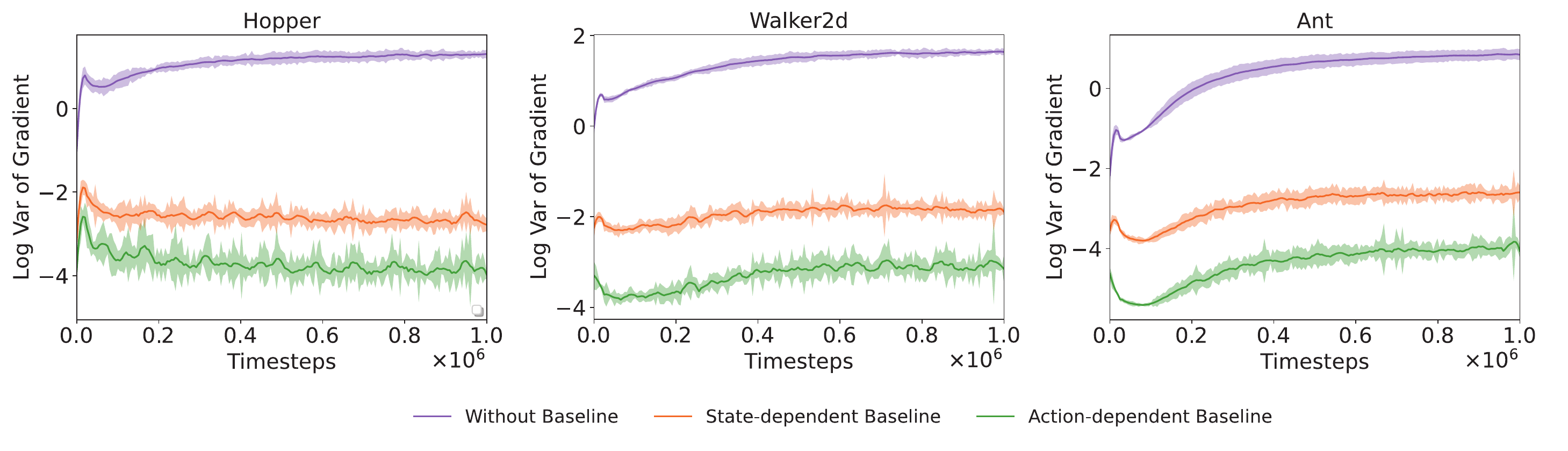}
		\vspace{-10pt}
		\caption{Results of our method with different baselines (without baseline, state-dependent baseline, action-dependent baseline) on Hopper, Walker2d, and Ant. The shaded region indicates the standard deviation over five random seeds. The $X$-aixs denotes the environment timesteps. The $Y$-axis denotes the logarithm of the gradient variance.}\label{variance_comparison}
		\label{two variants}
	\end{figure*}
	
	\begin{table*}[htb]
		%		\footnotesize
		\caption{Results of timesteps needed to achieve the return threshold on six representative continuous control tasks among these state-of-the-art methods. We indicate the best results in boldface in this table. $1000$ denotes that the method did not reach a threshold within $10^6$ timesteps.}
		\label{tasks_sample_threshold}
		%	\resizebox{\textwidth}{16mm}{
		\centering
		\begin{adjustbox}{width=1\textwidth}
			\begin{tabular}{l| c  c  c  c c c  c c c  || c}
				\hline
				\multicolumn{10}{c}{Timesteps to reach a threshold ($\times 10^3$)}\\
				\hline
				Timesteps& ACER & IMPALA & IPG & SAC & TD3 & PPO &  SLAC& PGAFB &Off-OAB & Return Threshold\\
				\hline
				HalfCheetah & $1000$  & $1000$ & $1000$ & $110$ & $70$  &  $120$ & $1000$ & $50$ & \bm{$50$}&  $5000$\\
				\hline
				Hopper &  $1000$ & $1000$ & $1000$ & $300$   & $160$ &  \bm{$80$} & $1000$ & $190$ & $180$ & $2000$ \\
				\hline
				Swimmer	& $1000$ & $1000$ & $800$  & $1000$  & $210$  &  \bm{$70$}  & $1000$  &  $90$&$170$ & $50$ \\
				\hline
				Walker2d &  $1000$ & $1000$  & $1000$  &  $680$  & $280$  & $750$  &  $1000$ & $120$ &\bm{$100$} & $3000$ \\
				\hline
				Ant & $1000$ & $1000$  &  $1000$ & $670$ & $450$ & $1000$ & $1000$ & $230$ &\bm{$100$}& $3000$ \\
				\hline    
				Humanoid & $1000$  & $1000$ & $1000$  & $870$  & $965$  & $1000$  & $1000$ & $1000$ & \bm{$420$} &  $5000$ \\
				\hline    
			\end{tabular}
			%	}
		\end{adjustbox}
	\end{table*}

	\subsection{Comparison with Other Baselines}
	To assess the effectiveness of the proposed action-dependent baseline, we compare its overall performance with that of other baselines, namely the state-dependent baseline and the case without any baseline.

	Figure \ref{comparison_with_other_baselines} illustrates the comparison results. As depicted, the proposed Off-OAB method, when integrated with baselines (state-dependent and action-dependent baseline), achieves the same returns in fewer timesteps compared to the no-baseline case across various tasks. Furthermore, it attains higher final returns than the no-baseline approach on six tasks. Notably, Figure \ref{comparison_with_other_baselines} highlights a distinct performance gap between Off-OAB with the action-dependent baseline and the state-dependent baseline. The proposed action-dependent baseline enables Off-OAB to reach comparable returns more efficiently than the state-dependent baseline. Overall, the results in Figure \ref{comparison_with_other_baselines} demonstrate the proposed action-dependent baseline outperforms both the state-dependent baseline and the case without any baseline.

	\subsection{Study on Sample Efficiency}
	We investigate the sample efficiency of the proposed Off-OAB by comparing the timesteps needed to reach a certain return against other state-of-the-art methods. 
	
	Table \ref{tasks_sample_threshold} shows the timesteps needed to meet specific return thresholds during training, with thresholds set at $5000$, $2000$, $50$, $3000$, $3000$, and $5000$ for HalfCheetah, Hopper, Swimmer, Walker2d, Ant, and Humanoid, respectively. These thresholds were selected based on the overall training performance of the compared methods. Our method consistently required fewer timesteps to reach these thresholds, demonstrating superior sample efficiency. Specifically, for HalfCheetah, Walker2d, Ant, and Humanoid, it needed only $50 \times 10^3$, $100 \times 10^3$, $100 \times 10^3$, and $420 \times 10^3$ timesteps, respectively. The data from Table \ref{tasks_sample_threshold} indicate our method's enhanced sample efficiency over other state-of-the-art methods.

	\subsection{Study on Variance Reduction}
	We evaluate the variance reduction capability of our action-dependent baseline by contrasting the gradient variance of our baseline against that of other baselines, including the state-dependent baseline and the case without any baseline.

	To analyze the gradient variance, we use the logarithm of gradient variance, which correlates positively with variance and effectively illustrates its magnitude. This gradient variance is calculated from ten repeated gradient measurements under the same policy. The logarithm of gradient variance during the training process is shown in Figure \ref{variance_comparison}. Our method with an action-dependent baseline consistently exhibits a lower logarithm of gradient variance compared to both the state-dependent baseline and the case without any baseline in representative tasks (Hopper, Walker2d, and Ant). Notably, the variance gap between our action-dependent baseline and the case without any baseline is significant in these tasks, as is the gap between our baseline and the state-dependent baseline. These findings, depicted in Figure \ref{variance_comparison}, empirically demonstrate our action-dependent baseline's effectiveness in reducing the gradient estimator's variance for off-policy methods.

	\section{Conclusion}
	In this paper, we present an off-policy policy gradient method equipped with an optimal action-dependent baseline (Off-OAB) to reduce the variance of the policy gradient estimator. This method provides a novel perspective for tackling the high variance challenge in off-policy policy gradient methods. We start by introducing an action-dependent baseline that is theoretically unbiased for the off-policy policy gradient (OPPG) estimator. From this baseline, we derive its optimal form to minimize OPPG's variance effectively. For practical application, we approximate this optimal baseline by adopting a similar action-dependent baseline and demonstrate this baseline is close to the optimal one. With this simplified optimal action-dependent baseline, we propose an off-policy policy gradient method with this baseline (Off-OAB) and provide the detailed algorithm of the proposed Off-OAB. Extensive experiments on OpenAI Gym and MuJoCo tasks demonstrate our method's effectiveness over state-of-the-art methods in most tasks. Our experiments show our method's superior sample efficiency and provide empirical evidence supporting our action-dependent baseline as an effective variance reduction strategy.

	\vspace{-20pt}
	\appendix 
	
	\section{Appendix}

	\subsection{Derivation of Eq. (\ref{gradient estimator with multivariate gaussian})}\label{gradient_estimator_appendix}

	Recall the off-policy policy gradient estimator in Eq. (\ref{expectation formulation-01}), we know
	\begin{flalign}
		\mathcal{g}_{\text{off}}&= \mathbb{E}_{s\sim  d_{\mu}, a\sim \mu(a|s)}\left[\rho(s,a) \nabla_{\theta}\log\pi(a|s)  Q_{\pi}(s,a) \right] \nonumber \\
		&= \mathbb{E}_{s\sim  d_{\mu}, a\sim \mu(a|s)}\left[ \rho(s,a)  \nabla_{\theta}\log \left(\prod_{i=1}^{m} \pi(a^i |s)\right)  Q_{\pi}(s,a) \right]  \nonumber \\
		&=\mathbb{E}_{s\sim  d_{\mu}, a\sim  \mu(a|s)} \left[\sum_{i=1}^{m} \rho(s,a) \nabla_{\theta}\log\pi(a^i|s)  Q_{\pi}(s,a)   \right],
	\end{flalign}
	which concludes the result in Eq. (\ref{gradient estimator with multivariate gaussian}).

	\subsection{Unbiasedness of Action-dependent Baseline}\label{Unbiassness}
	
	\noindent\textbf{Remark \ref{remark_unbiasedness}} (Unbiasedness of Action-dependent Baseline)
	\emph{
		The action-dependent baseline in the proposed OPPG estimator, as shown in Eq. (\ref{non-biased estimator}), maintains the unbiased nature of the original OPPG estimator in Eq. (\ref{gradient estimator with multivariate gaussian}), supported by the fact that,
		\begin{flalign}
			\mathbb{E}_{a \sim \mu(a|s)}\left[ \rho(s, a)\nabla_{\theta}\log\pi(a^i |s) b_{i}(s, a ^{-i} ) \right]=0,
		\end{flalign}
		which implies the following fact
		\begin{flalign}
			\mathcal{g}_{\text{off}} =\mathcal{g}_{\text{off}}(b).
		\end{flalign}
	}
	\begin{proof}
		Recall the proposed action-dependent baseline $b_{i}(s, a^{-i})$ in Eq. (\ref{non-biased estimator}), we know,
		\begin{flalign} \label{app-eq-01}
			\begin{aligned}
				&\mathbb{E}_{a \sim \mu(a|s)}\left[ \frac{\pi(a|s) }{\mu(a|s)}\nabla_{\theta}\log\pi(a^i |s) b_{i}(s, a ^{-i} ) \right]    \\
				=&\int_{a\in A}\bcancel{\mu(a|s)}\dfrac{\pi(a|s) }{\bcancel{\mu(a|s)}}\nabla_{\theta}\log\pi(a^i |s) b_{i}(s, a ^{-i} ) \mathrm{d}a
				\\
				=&\int_{a\in A}\pi(a|s) \nabla_{\theta}\log\pi(a^i |s) b_{i}(s, a ^{-i} ) \mathrm{d}a
				\\
				=& \mathbb{E}_{a \sim \pi}\left[ \nabla_{\theta}\log\pi(a^i |s) b_i(s, a^{-i})  \right]  \\
				=& \mathbb{E}_{a^{-i} \sim \pi }\left[ \mathbb{E}_{a^i \sim \pi } \left[  \nabla_{\theta}\log\pi(a^i |s) b_i(s, a^{-i})   \right] \right] \\
				=& \mathbb{E}_{a^{-i} \sim \pi}\left[ \sum_{a^i} \pi(a^i |s) \frac{\nabla_{\theta}\pi(a^i |s)}{\pi(a^i |s)} b_i(s, a^{-i}) \right] \\
				=& \mathbb{E}_{a^{-i} \sim \pi }\left[ \nabla_{\theta} \left(\sum_{a^i}\pi(a^i|s) b_i(s, a^{-i} )  \right)  \right ] \\
				=& \mathbb{E}_{a^{-i} \sim \pi }\left[ \nabla_{\theta} \left( b_i(s, a^{-i} ) \right)  \right ] =0.  
			\end{aligned} 
		\end{flalign}
		Furthermore, recall OPPG estimator in Eq. (\ref{non-biased estimator}), we know
		\begin{flalign}
			\nonumber
			&\mathcal{g}_{\text{off}}(b)\\
			& =  \mathbb{E}_{s\sim d_{\mu}, a\sim \mu}\bigg[\sum_{i=1}^m  \frac{\pi(a|s)}{\mu(a|s)} \nabla_{\theta}\log\pi(a^i|s)  ( Q_{\pi}(s,a) \\
			& \qquad \qquad\qquad\qquad\qquad\qquad\qquad\qquad  - b_i(s, a^{-i}))\bigg]\\
			&\overset{(\ref{app-eq-01})}{=} \sum_{i=1}^m  \mathbb{E}_{s\sim d_{\mu}, a\sim \mu}\left[ \dfrac{\pi(a|s)}{\mu(a|s)} \nabla_{\theta}\log\pi(a^i|s)  Q_{\pi}(s,a) \right]  	 \\
			& -
			\sum_{i=1}^m  \mathbb{E}_{s\sim d_{\mu}}\bcancel{\mathbb{E}_{a\sim \mu}\left[ \dfrac{\pi(a|s)}{\mu(a|s)} \nabla_{\theta}\log\pi(a^i|s) b_i(s, a^{-i})\right]} \label{app-eq-02}
			\\
			& =\sum_{i=1}^m  \mathbb{E}_{s\sim d_{\mu}, a\sim \mu}\left[ \dfrac{\pi(a|s)}{\mu(a|s)} \nabla_{\theta}\log\pi(a^i|s)  Q_{\pi}(s,a) \right]=\mathcal{g}_{\text{off}}. \label{no_bias_proof}
		\end{flalign}
		The result in Eq. (\ref{no_bias_proof}) illustrates that the proposed action-dependent baseline $b_{i}(s, a^{-i})$ does not introduce bias into OPPG estimator in Eq. (\ref{gradient estimator with multivariate gaussian}).
	\end{proof}

	\subsection{Optimal Off-Policy Action-Dependent Baseline}\label{optimal_action_dependent_baseline_proof}

	\noindent\textbf{Theorem \ref{theorem-optimal-baseline}} (Optimal Off-Policy Action-Dependent Baseline)\textbf{.}
	\emph{
		Let $\mathcal{g}_{\text{off}}(b)$ be the off-policy policy gradient estimator defined in Eq. (\ref{non-biased estimator}). The optimal off-policy action-dependent baseline that minimizes the variance of $\mathcal{g}_{\text{off}}(b)$ is
		\begin{flalign}
			b^{\star}_i(s, a^{-i}) =\frac{\mathbb{E}_{a^i \sim \mu} \left[ \lVert \rho(s, a) \nabla_{\theta}\log\pi(a^i|s) \rVert^2 Q_{\pi}(s,a) \right]}{\mathbb{E}_{a^i \sim \mu} \left[ \lVert \rho(s, a) \nabla_{\theta}\log\pi(a^i|s) \rVert^2 \right]},
		\end{flalign}
		where $\rho(s, a)=\frac{\pi(a|s)}{\mu(a|s)}$.
	}
	\begin{proof}
		We expand the variance of $\mathcal{g}_{\text{off}}(b)$ in Eq. (\ref{non-biased estimator}) as follows,
		\begin{flalign}\label{variance formulation}
			\begin{aligned}
				& \mathrm{Var}\left[ \mathcal{g}_{\text{off}}(b)\right] = \mathrm{Var}\left[\sum_{i=1}^m \mathcal{g}_{\text{off}}^i(b) \right] \\
				&=\sum_{i=1}^m \mathrm{Var}( \mathcal{g}_{\text{off}}^i(b)) + \sum_{i=1}^{m} \sum_{j \neq i}\mathrm{Cov}\big( \mathcal{g}_{\text{off}}^i(b), \mathcal{g}_{\text{off}}^j(b) \big )  \\
				& =\sum_{i=1}^{m} \mathrm{Var}( \mathcal{g}_{\text{off}}^i(b)) + \sum_{i=1}^m \sum_{j \neq i }\Big( \mathbb{E}_{s\sim  d_{\mu}, a\sim \mu}\big[ \mathcal{g}_{\text{off}}^i(b)^{\top} \mathcal{g}_{\text{off}}^j(b)\big] \\
				& \qquad\qquad -  \mathbb{E}_{s\sim  d_{\mu}, a \sim \mu}\big[  \mathcal{g}_{\text{off}}^i(b) \big]^{\top} \mathbb{E}_{s \sim  d_{\mu}, a\sim \mu}\big[ \mathcal{g}_{\text{off}}^j(b) \big]  \Big ).
			\end{aligned}
		\end{flalign}
		Recall the following fact,
		\begin{flalign} \label{assumption}
			\nabla_{\theta}\log\pi(a^i |s)^{\top} \nabla_{\theta} \log \pi(a^j |s) \approx 0, \forall i \neq j.
		\end{flalign} 
		The assumption in Eq.(\ref{assumption}) means that different subsets of parameters strongly influence different action dimensions of factors, which is commonly used. Under such assumption, $\mathbb{E}_{s\sim  d_{\mu}, a\sim \mu}[ \mathcal{g}_{\text{off}}^i(b)^{\top} \mathcal{g}_{\text{off}}^j(b)]$ in Eq.(\ref{variance formulation}) satisfies: 
		\begin{flalign}\label{zero expectation}
			\mathbb{E}_{s\sim  d_{\mu}, a\sim \mu}[ \mathcal{g}_{\text{off}}^i(b)^{\top} \mathcal{g}_{\text{off}}^j(b)]=0. 
		\end{flalign}

		Furthermore, we define some addition notations as follows:
		\begin{flalign}
			%		\nonumber
			& \rho(s, a)=\frac{\pi(a|s)}{\mu(a|s)}, \\
			&z_i := \rho(s, a) \nabla_{\theta}\log\pi(a^i|s),~~~z_j :=\rho(s, a) \nabla_{\theta}\log\pi(a^j|s),\\
			%		\nonumber
			&M_{i,j} := \mathbb{E}_{s\sim  d_{\mu}, a \sim \mu}[z_i Q_{\pi}(s,a) ]^{\top} \mathbb{E}_{s \sim  d_{\mu}, a\sim \mu}[z_j Q_{\pi}(s,a)].
			%		&\rho(s, a)=\frac{\pi(a|s)}{\mu(a|s)}, \\
			%		&M_{i,j}\\
			%		&=: \mathbb{E}_{s\sim  d_{\mu}, a \sim \mu}[\rho(s, a) \log\pi(a^i|s) Q_{\pi}(s,a) ]^{\top} \mathbb{E}_{s \sim  d_{\mu}, a\sim \mu}[z_j Q_{\pi}(s,a)]
		\end{flalign}
		By substituting Eq.(\ref{zero expectation}) into Eq.(\ref{variance formulation}), the variance $ \mathrm{Var}[ g_b(\theta)]$ in Eq.(\ref{variance formulation}) can be rewritten as: 
		\begin{small}
			\begin{flalign}
				%	&\mathrm{Var}\Big[ g_b(\theta)\Big] \nonumber \\
				&\mathrm{Var}\left[ \mathcal{g}_{\text{off}}(b)\right]  \\
				&=\sum_{i=1}^{m} \mathrm{Var}( \mathcal{g}_{\text{off}}^i(b))   -  \sum_{i=1}^m \sum_{j \neq i } \mathbb{E}_{s\sim  d_{\mu}, a \sim \mu}\big[  \mathcal{g}_{\text{off}}^i(b) \big]^{\top} \mathbb{E}_{s \sim  d_{\mu}, a\sim \mu}\big[ \mathcal{g}_{\text{off}}^j(b) \big]  \nonumber\\
				&= \sum_{i=1}^{m} \mathrm{Var}( \mathcal{g}_{\text{off}}^i(b)) - \sum_{i=1}^m \sum_{j \neq i } M_{i, j} ,\label{final variance} 
				%	& \text{where} \nonumber \\
				%	& \quad M_{i,j}= \mathbb{E}_{s\sim  d_{\mu}, a \sim \mu}[z_i Q_{\pi}(s,a) ]^{\top} \mathbb{E}_{s \sim  d_{\mu}, a\sim \mu}[z_j Q_{\pi}(s,a)] \\
				%	&\qquad z_i= \frac{\pi(a|s)}{\mu(a|s)} \nabla_{\theta}\log\pi(a^i|s) \label{zi} \\
				%	&\qquad z_j=\frac{\pi(a|s)}{\mu(a|s)} \nabla_{\theta}\log\pi(a^j|s). \label{zj}
			\end{flalign}
		\end{small}
		where the last equation holds due to the unbiasedness of the action-dependent baseline in Remark \ref{remark_unbiasedness}, i.e.,
		\begin{flalign}
			\mathbb{E}_{a \sim \mu(a|s)}\left[ \rho(s,a) \nabla_{\theta}\log\pi(a^i |s) b_{i}(s, a ^{-i} ) \right]=0.
		\end{flalign}

		To derive the optimal action-dependent baseline, our goal is to minimize the variance of the OPPG estimator in Eq. (\ref{final variance}). Notably, the second term of Eq. (\ref{final variance}) is independent of the action-dependent baseline $b_{i}(s, a^{-i})$. Therefore, the optimization focuses on minimizing the first term of Eq. (\ref{final variance}), simplifying our problem to the following optimization problem:
		\begin{flalign}
			\min_{\{g_{i}\}_{i=1,\cdots,m}}\sum_{i=1}^{m} \mathrm{Var}( \mathcal{g}_{\text{off}}^i(b)).
		\end{flalign}

		To minimize the summation in the first term, each component ($\mathrm{Var}( \mathcal{g}_{\text{off}}^i(b))$) in this summation needs to be minimized.

		For a baseline $b_{i}(s,a^{-i})$, we have
		\begin{small}
			\begin{flalign}\label{separate variance}
				&\mathrm{Var}\left[\mathcal{g}_{\text{off}}^i(b)\right] \\
				& = \mathrm{Var}\left[(\rho(s, a) \nabla_{\theta}\log\pi(a^i|s) )(Q_{\pi}(s,a)  - b_{i}(s,a^{-i})) \right]\\
				&=\mathbb{E}_{s\sim  d_{\mu}, a\sim \mu}\left[ \lVert \rho(s, a) \nabla_{\theta}\log\pi(a^i|s) \rVert^2 (Q_{\pi}(s,a) - b_{i}(s,a^{-i}))^2\right]   \\
				& -  \Big( \mathbb{E}_{s\sim  d_{\mu}, a\sim \mu}[(\rho(s, a) \nabla_{\theta}\log\pi(a^i|s) )(Q_{\pi}(s,a)  - b_{i}(s,a^{-i}))] \Big ) ^2  \\
				&\overset{(\ref{action_dependent_equation})}= \mathbb{E}_{ s\sim  d_{\mu}, a\sim \mu}[\lVert \rho(s, a) \nabla_{\theta}\log\pi(a^i|s) \rVert^2 (Q_{\pi}(s,a))^2 ]\\
				& \qquad - 2 \mathbb{E}_{s\sim  d_{\mu}, a\sim \mu}[ \lVert \rho(s, a) \nabla_{\theta}\log\pi(a^i|s) \rVert^2 b_i(s,a^{-i})Q_{\pi}(s,a) ] \\
				& \qquad + \mathbb{E}_{s\sim  d_{\mu}, a\sim \mu}[\lVert \rho(s, a) \nabla_{\theta}\log\pi(a^i|s) \rVert^2  b_i(s,a^{-i})^2] \\
				& \qquad -  \Big(  \mathbb{E}_{s\sim  d_{\mu}, a\sim \mu}[(\rho(s, a) \nabla_{\theta}\log\pi(a^i|s) ) Q_{\pi}(s,a)] \Big)^2 \\
				&= \mathbb{E}_{s\sim  d_{\mu}, a^{-i}\sim \mu} \Big[   b_i(s,a^{-i})^2 \mathbb{E}_{a^i \sim \mu}[\lVert \rho(s, a) \nabla_{\theta}\log\pi(a^i|s) \rVert^2  ]  \\
				&\qquad - 2 b_i(s,a^{-i}) \mathbb{E}_{a^i \sim \mu}[ \lVert \rho(s, a) \nabla_{\theta}\log\pi(a^i|s) \rVert^2 Q_{\pi}(s,a) ] \Big] \\
				&\qquad + \mathbb{E}_{ s\sim  d_{\mu}, a\sim \mu}\left[\lVert \rho(s, a) \nabla_{\theta}\log\pi(a^i|s) \rVert^2 (Q_{\pi}(s,a))^2 \right]  \label{Q^2}  \\
				& \qquad -  \Big(  \mathbb{E}_{s\sim  d_{\mu}, a\sim \mu}[(\rho(s, a) \nabla_{\theta}\log\pi(a^i|s) ) Q_{\pi}(s,a)] \Big)^2  \label{total^2}.
			\end{flalign}
		\end{small}
		
		The final terms in Eq. (\ref{Q^2}) and (\ref{total^2}) are independent of the baseline $b_i(s,a^{-i})$, which can be treated as constants. Therefore, we can simplify the variance $\mathrm{Var}[ \mathcal{g}_{\text{off}}^i(b)]$  as follows,
		\begin{flalign}
			\begin{aligned}\label{final_variance_estimation}
				&\mathrm{Var}\left[ \mathcal{g}_{\text{off}}^i(b) \right]\\
				&= \mathbb{E}_{s\sim  d_{\mu}, a^{-i}\sim \mu} \Big[   b_i(s,a^{-i})^2 \mathbb{E}_{a^i \sim \mu}[\lVert \rho(s, a) \nabla_{\theta}\log\pi(a^i|s) \rVert^2  ]  \\
				&\qquad - 2 b_i(s,a^{-i}) \mathbb{E}_{a^i \sim \mu}[ \lVert \rho(s, a) \nabla_{\theta}\log\pi(a^i|s) \rVert^2 Q_{\pi}(s,a) ] \Big] \\
				& \qquad\qquad+ constant,
			\end{aligned}
		\end{flalign}
		where
		\begin{flalign}
			\begin{aligned}
				&constant =\\
				&\mathbb{E}_{s\sim  d_{\mu}, a\sim \mu}[ \lVert \rho(s, a) \nabla_{\theta}\log\pi(a^i|s) \rVert^2 (Q_{\pi}(s,a))^2 ] \\
				& -  \Big(  \mathbb{E}_{s\sim  d_{\mu}, a\sim \mu}[(\rho(s, a) \nabla_{\theta}\log\pi(a^i|s) ) Q_{\pi}(s,a)] \Big)^2.  
			\end{aligned}
		\end{flalign}
		
		The optimal action-dependent baseline is found by minimizing the variance in (\ref{final_variance_estimation}). 
		The optimal baseline is derived by
		\begin{flalign}
			\nabla_{b_i}\mathrm{Var}[ \mathcal{g}_{\text{off}}^i(b)]=0,
		\end{flalign}
		which concludes
		\begin{flalign}\label{optimal action-dependent baseline}
			b^{\star}_i(s, a^{-i}) =\frac{\mathbb{E}_{a^i \sim \mu} \left[\lVert \rho(s, a) \nabla_{\theta}\log\pi(a^i|s) \rVert^2 Q_{\pi}(s,a) \right]}{\mathbb{E}_{a^i \sim \mu} \left[\lVert \rho(s, a) \nabla_{\theta}\log\pi(a^i|s) \rVert^2 \right]}.
		\end{flalign}
		The notation $b^{\star}_i(s, a^{-i})$  in Eq. (\ref{optimal action-dependent baseline})  represents the derived optimal action-dependent baseline for $i$-th action dimension. 
	\end{proof}
	
	\subsection{Variance Difference: $b_i(s, a^{-i})$ vs.  $b^{\star}_i(s, a^{-i})$ }\label{excess_variance_proof}

	\begin{lemma}\label{baseline_excess_variance_theorem}
		The variance difference of OPPG estimator between $b_i(s, a^{-i})$ and $b^{\star}_i(s, a^{-i})$ in the $i$-th action dimension can be defined as,
		\begin{flalign}
			\begin{aligned}
				&\Delta \mathrm{Var}(\mathcal{g}_{\text{off}}^i(b_i(s, a^{-i})) )  \\
				&\triangleq  \mathrm{Var}\left[ \mathcal{g}_{\text{off}}^i(b_i(s, a^{-i})) \right]- \mathrm{Var} \left[ \mathcal{g}_{\text{off}}^i(b^{\star}_i(s, a^{-i})) \right], 
			\end{aligned}
		\end{flalign} 
		and it satisfies
		\begin{flalign}\label{Delta_var_bi}
			\begin{aligned}
				&\Delta \mathrm{Var}(\mathcal{g}_{\text{off}}^i(b_i(s, a^{-i})) )  \\
				&=\mathbb{E}_{s\sim  d_{\mu}, a^{-i}\sim \mu} \Big[ \Big( b_i(s, a^{-i}) - b^{\star}_i(s, a^{-i}) \Big)^2 \\
				&\qquad\qquad\qquad\qquad\mathbb{E}_{a^i \sim \mu}\left[ \lVert \rho(s, a) \nabla_{\theta}\log\pi(a^i|s) \rVert^2 \right]\Big],
			\end{aligned}
		\end{flalign}
		where $b_i(s, a^{-i})$ represents an action-dependent baseline or a state-dependent baseline. 
	\end{lemma}
	\begin{proof}
		Recall the formulation of variance in Eq. (\ref{final_variance_estimation}), we have 
		\begin{flalign}\label{var_state_baseline}
			\begin{aligned}
				&\mathrm{Var}\left[ \mathcal{g}_{\text{off}}^i(b_i(s, a^{-i})) \right]\\
				&= \mathbb{E}_{s\sim  d_{\mu}, a^{-i}\sim \mu} \Big[   b_i(s,a^{-i})^2 \mathbb{E}_{a^i \sim \mu}[\lVert \rho(s, a) \nabla_{\theta}\log\pi(a^i|s) \rVert^2  ]  \\
				&\qquad - 2 b_i(s,a^{-i}) \mathbb{E}_{a^i \sim \mu}[ \lVert \rho(s, a) \nabla_{\theta}\log\pi(a^i|s) \rVert^2 Q_{\pi}(s,a) ] \Big] \\
				& \qquad+ constant.
			\end{aligned}
		\end{flalign}
		Let the baseline be the optimal action-dependent baseline ($b^{\star}_i(s, a^{-i}) $ ), the OPPG variance in Eq. (\ref{final_variance_estimation}) becomes:
		\begin{flalign}\label{optimal_baseline_variance}
			\begin{aligned}
				&\mathrm{Var}_{a^i \sim \mu}[ \mathcal{g}_{\text{off}}^i(b^{\star}_i(s, a^{-i}))]\\
				& = \mathbb{E}_{s\sim  d_{\mu}, a^{-i}\sim \mu} \Big[  b^{\star}_i(s,a^{-i})^2 \mathbb{E}_{a^i \sim \mu}[\lVert \rho(s, a) \nabla_{\theta}\log\pi(a^i|s) \rVert^2  ]  \\
				& \quad- 2 b^{\star}_i(s,a^{-i}) \mathbb{E}_{a^i \sim \mu}[ \lVert \rho(s, a) \nabla_{\theta}\log\pi(a^i|s) \rVert^2 Q_{\pi}(s,a) ] \Big] \\
				& \qquad + constant\\
				& \overset{(\ref{them-baseline-01})}{=}	\mathbb{E}_{s\sim  d_{\mu}, a^{-i}\sim \mu} \Big[  b^{\star}_i(s,a^{-i})^2 \mathbb{E}_{a^i \sim \mu}[\lVert \rho(s, a) \nabla_{\theta}\log\pi(a^i|s) \rVert^2  ]  \\
				& \quad- 2 b^{\star}_i(s,a^{-i})^2 \mathbb{E}_{a^i \sim \mu}[\lVert \rho(s, a) \nabla_{\theta}\log\pi(a^i|s) \rVert^2  ] \Big] \\
				& \qquad+ constant\\
				&= \mathbb{E}_{s\sim  d_{\mu}, a^{-i}\sim \mu}\left[- b^{\star}_i(s,a^{-i})^2 \mathbb{E}_{a^i \sim \mu}[\lVert \rho(s, a) \nabla_{\theta}\log\pi(a^i|s) \rVert^2  ] \right]\\
				&\qquad +constant.
			\end{aligned}
		\end{flalign}
		According to Eq. (\ref{var_state_baseline}) and Eq. (\ref{optimal_baseline_variance}), we have 
		\begin{flalign}
			\small
			\begin{aligned}
				&\Delta \mathrm{Var}(\mathcal{g}_{\text{off}}^i(b_i(s, a^{-i})) )\\
				&=\mathrm{Var}\left[ \mathcal{g}_{\text{off}}^i(b_i(s, a^{-i}) ) \right]- \mathrm{Var} \left[ \mathcal{g}_{\text{off}}^i(b^{\star}_i(s,a^{-i})) \right] \\
				&= \mathbb{E}_{s\sim  d_{\mu}, a^{-i}\sim \mu} \Big[ b_i(s, a^{-i})^2 \mathbb{E}_{a^i \sim \mu}[\lVert \rho(s, a) \nabla_{\theta}\log\pi(a^i|s) \rVert^2  ]  \\
				& \quad- 2 b_i(s, a^{-i}) b^{\star}_i(s,a^{-i}) \mathbb{E}_{a^i \sim \mu}[\lVert \rho(s, a) \nabla_{\theta}\log\pi(a^i|s) \rVert^2  ] \\
				&\quad + b^{\star}_i(s,a^{-i})^2 \mathbb{E}_{a^i \sim \mu}[\lVert \rho(s, a) \nabla_{\theta}\log\pi(a^i|s) \rVert^2  ]\Big] \\
				&=\mathbb{E}_{s\sim  d_{\mu}, a^{-i}\sim \mu} \Big[ \Big( b_i(s, a^{-i}) - b^{\star}_i(s, a^{-i}) \Big)^2 \\
				&\qquad\qquad\qquad\qquad\qquad\quad\mathbb{E}_{a^i \sim \mu}\left[ \lVert \rho(s, a) \nabla_{\theta}\log\pi(a^i|s) \rVert^2 \right]\Big].
			\end{aligned}
		\end{flalign}
		
	\end{proof}
	
	\subsection{Variance Difference between Optimal State and Action-Dependent Baseline}\label{optimal_state_dependent_excess_variance}
	%\begin{theorem}\label{excess_variance_theorem}
	\noindent\textbf{Theorem \ref{excess_variance_theorem}} (Variance Difference between Optimal State and Action-Dependent Baseline)\textbf{.}
	\emph{		Let the variance difference $\Delta \mathrm{Var}(\mathcal{g}_{\text{off}}^i(b^{\star}(s))) $ between optimal state and action-dependent baseline for OPPG estimator be defined in Eq. (\ref{variance_diff_definition}). This variance difference satisfies,
		\begin{flalign}
			\begin{aligned}
				& \Delta \mathrm{Var}(\mathcal{g}_{\text{off}}^i(b^{\star}(s)))\\
				&= \mathbb{E}_{s\sim  d_{\mu}, a^{-i}\sim \mu}\Big[ \Big( b^{\star}(s)- b^{\star}_i(s, a^{-i}) \Big)^2 \\
				&\qquad\qquad\qquad\qquad \mathbb{E}_{a^i \sim \mu}\left[ \lVert \rho(s, a) \nabla_{\theta}\log\pi(a^i|s) \rVert^2 \right]\Big].
			\end{aligned}
		\end{flalign}
	}
	\begin{proof}
		When $b_i(s, a^{-i})$ is the optimal state-dependent baseline $b^{\star}(s)$, referencing notations in Appendix \ref{excess_variance_proof} and the formulation in Eq. (\ref{Delta_var_bi}) from Lemma \ref{baseline_excess_variance_theorem}, we have:
		\begin{flalign}
			\begin{aligned}
				& \Delta \mathrm{Var}(\mathcal{g}_{\text{off}}^i(b^{\star}(s)))\\
				&=\mathrm{Var}_{a^i \sim \mu}\left[ \mathcal{g}_{\text{off}}^i(b^{\star}(s)) \right]- \mathrm{Var}_{a^i \sim \mu} \left[ \mathcal{g}_{\text{off}}^i(b^{\star}_i(s, a^{-i})) \right]  \\
				&= \mathbb{E}_{s\sim  d_{\mu}, a^{-i}\sim \mu}\Big[ \Big( b^{\star}(s)- b^{\star}_i(s, a^{-i}) \Big)^2 \\
				&\qquad\qquad\qquad\quad \mathbb{E}_{a^i \sim \mu}\left[ \lVert \rho(s, a) \nabla_{\theta}\log\pi(a^i|s) \rVert^2 \right]\Big].
			\end{aligned}
		\end{flalign}
	\end{proof}
	
	\subsection{Close to Optimal Action-Dependent Baseline}\label{excess_variance_Q_pi_sa}
	\noindent\textbf{Theorem \ref{theorem_Q_pi_sa}} (Close to Optimal Action-Dependent Baseline)\textbf{.}
	\emph{Define the approximated baseline as $\mathbb{E}_{a^i \sim \mu}[Q_{\pi}(s,a)]$ in Eq. (\ref{appro_b_s_a}). 
		The variance of this approximated baseline is close to that of the optimal action-dependent baseline when the policy factor is weakly correlated with the action value,
		\begin{flalign}	
	\begin{aligned}
		&\Delta \mathrm{Var}( \mathcal{g}_{\text{off}}^i(\mathbb{E}_{a^i \sim \mu}[Q_{\pi}(s,a)])) \\
		&= \mathrm{Var}\left[ \mathcal{g}_{\text{off}}^i(\mathbb{E}_{a^i \sim \mu}[Q_{\pi}(s,a)]) \right]- \mathrm{Var} \left[ \mathcal{g}_{\text{off}}^i(b^{\star}_i(s, a^{-i})) \right] 
		\approx 0, \\
		& \text{when},\\
		& \quad \mathbb{E}_{a^i \sim \mu} \left[ \lVert \rho(s, a) \nabla_{\theta}\log\pi(a^i|s) \rVert^2 Q_{\pi}(s,a) \right] \\
		&\quad \approx \mathbb{E}_{a^i \sim \mu} \left[ \lVert \rho(s, a) \nabla_{\theta}\log\pi(a^i|s) \rVert^2  \right] \mathbb{E}_{a^i \sim \mu}[Q_{\pi}(s,a)].
	\end{aligned}
\end{flalign}
	}
	\begin{proof}
		As outlined in Eq. (\ref{Delta_var_bi}) from Lemma \ref{baseline_excess_variance_theorem} in Appendix \ref{excess_variance_proof}, when $b_i(s, a^{-i})$ is defined as $\mathbb{E}_{a^i \sim \mu}[Q_{\pi}(s,a)]$, the variance difference between $\mathbb{E}_{a^i \sim \mu}[Q_{\pi}(s,a)]$ and the optimal baseline $b^{*}_i(s, a^{-i}) $in the $i$-th action dimension is expressed as:
		\small
		\begin{flalign}
			\small
			\begin{aligned}\label{excess_variance_E_Q_eq}
				&\Delta \mathrm{Var}( \mathcal{g}_{\text{off}}^i(\mathbb{E}_{a^i \sim \mu}[Q_{\pi}(s,a)])) \\
				&= \mathbb{E}_{s\sim  d_{\mu}, a^{-i}\sim \mu}\Big[ \\
				&\Big( \mathbb{E}_{a^i \sim \mu}[Q_{\pi}(s,a)]- b^{\star}_i(s, a^{-i}) \Big)^2 \mathbb{E}_{a^i \sim \mu}\left[ \lVert \rho(s, a) \nabla_{\theta}\log\pi(a^i|s) \rVert^2 \right] \Big]\\
				&\overset{(\ref{them-baseline-01})}{=}\mathbb{E}_{s\sim  d_{\mu}, a^{-i}\sim \mu}\Big[\\
				&\Big( \mathbb{E}_{a^i \sim \mu}[Q_{\pi}(s,a)] - \frac{\mathbb{E}_{a^i \sim \mu} \left[ \lVert \rho(s, a) \nabla_{\theta}\log\pi(a^i|s) \rVert^2 Q_{\pi}(s,a) \right]}{\mathbb{E}_{a^i \sim \mu} \left[ \lVert \rho(s, a) \nabla_{\theta}\log\pi(a^i|s) \rVert^2 \right]}\Big)^2 \\
				& \qquad\qquad\qquad\qquad\qquad\qquad\qquad \mathbb{E}_{a^i \sim \mu}\left[ \lVert \rho(s, a) \nabla_{\theta}\log\pi(a^i|s) \rVert^2 \right] \Big].
			\end{aligned}
		\end{flalign}
		When 
		\begin{flalign}
			\begin{aligned}
				& \mathbb{E}_{a^i \sim \mu} \left[ \lVert \rho(s, a) \nabla_{\theta}\log\pi(a^i|s) \rVert^2 Q_{\pi}(s,a) \right] \\
				&\approx \mathbb{E}_{a^i \sim \mu} \left[ \lVert \rho(s, a) \nabla_{\theta}\log\pi(a^i|s) \rVert^2  \right] \mathbb{E}_{a^i \sim \mu}[Q_{\pi}(s,a)],
			\end{aligned}
		\end{flalign}
		the variance difference in Eq. (\ref{excess_variance_E_Q_eq}) becomes:
		\begin{flalign}
			\begin{aligned}
				&\Delta \mathrm{Var}( \mathcal{g}_{\text{off}}^i(\mathbb{E}_{a^i \sim \mu}[Q_{\pi}(s,a)]))\\
				&\approx \mathbb{E}_{s\sim  d_{\mu}, a^{-i}\sim \mu}\Big[\Big( \mathbb{E}_{a^i \sim \mu}[Q_{\pi}(s,a)] - \mathbb{E}_{a^i \sim \mu}[Q_{\pi}(s,a)] \Big)^2  \\
				&\qquad\qquad\qquad\qquad\qquad\mathbb{E}_{a^i \sim \mu}\left[ \lVert \rho(s, a) \nabla_{\theta}\log\pi(a^i|s) \rVert^2 \right] \Big]\\
				&=0.
			\end{aligned}
		\end{flalign}
	\end{proof}
	
	\subsection{Experimental Details}\label{detailed_hy}
	For implementing benchmark methods (ACER, IMPALA, IPG, SAC, TD3, PPO, SLAC, and PGAFB) in our experiments, we utilize the authors' code or their hyperparameters from their respective publications. Our method's hyperparameters are detailed in Table \ref{hyperparameter_table}.
	\begin{table}[htb]
		\centering
		\caption{Hyperparameters used in our method.}\label{hyperparameter_table}
		%		\begin{adjustbox}{width=1\textwidth}
		\begin{tabular}{l |l}
			Parameter & Value\\
			\hline
			Discount ($\gamma$) &  $0.99$ \\ \hline
			Critic learning rate ($\lambda_{Q}$) & $0.0003$ \\ \hline
			Actor learning rate ($\lambda_{\pi}$) &$0.0003$ \\ \hline
			decay rate ($\tau$) & $0.004$ \\ \hline
			Replay buffer size & $10^6$ \\ \hline
			Batch size &  $256$ \\ \hline
			Number of hidden layers &  $2$ \\ \hline
			Number of hidden units per layer & $256$ \\ \hline
			Nonlinearity & ReLU  \\ \hline
			Target update interval & $250$ \\ \hline
			Timesteps before training & $25\times 10^3$ \\ \hline
			Episode number during evaluation & $10$ \\ \hline
			Timesteps of evaluation frequency & $10^4$ \\ \hline
			Maximal timesteps & $10^6$ 
		\end{tabular}
		%		\end{adjustbox}
	\end{table}

\vspace{-15pt}

	\bibliographystyle{IEEEtran}
	\bibliography{reference}
	
	%\newpage
	%If you have an EPS/PDF photo (graphics package needed), extra braces are
	% your own custom macro containing the $\backslash${\tt{includegraphics}} command to make things
	%\bf{If you include a photo:}\vspace{-33pt}
	%\begin{IEEEbiography}[{\includegraphics[width=1in,height=1.25in,clip,keepaspectratio]{fig1}}]{Michael Shell}
	%Use $\backslash${\tt{begin\{IEEEbiography\}}} and then for the 1st argument use $\backslash${\tt{includegraphics}} to declare and link the author photo.
	%Use the author's name as the 3rd argument followed by the biography text.
	%\end{IEEEbiography}
	%\vfill
	
\end{document}